\documentclass{article}

     \usepackage[preprint]{neurips_2020}

\usepackage[utf8]{inputenc} 
\usepackage[T1]{fontenc}    
\usepackage{hyperref}       
\usepackage{url}            
\usepackage{booktabs}       
\usepackage{amsfonts}       
\usepackage{nicefrac}       
\usepackage{microtype}      
\usepackage{natbib}
\usepackage{graphicx}
\usepackage{caption}
\usepackage{subcaption}
\usepackage{amsmath, algorithm, algcompatible}
\usepackage{multirow}
\usepackage{color}
\usepackage{xcolor}
\usepackage{amsfonts,amssymb}
\usepackage{makecell}
\usepackage{amsthm}

\usepackage[symbol]{footmisc}
\renewcommand{\thefootnote}{\fnsymbol{footnote}}

\newtheorem{innercustomthm}{Theorem}
\newenvironment{customthm}[1]
  {\renewcommand\theinnercustomthm{#1}\innercustomthm}
  {\endinnercustomthm}

\newtheorem{innercustomcorollary}{Corollary}
\newenvironment{customcorollary}[1]
  {\renewcommand\theinnercustomcorollary{#1}\innercustomcorollary}
  {\endinnercustomcorollary}

\usepackage{dsfont}
\usepackage{braket}
\algnewcommand\INPUT{\item[\textbf{Input:}]}%
\algnewcommand\OUTPUT{\item[\textbf{Output:}]}%
\usepackage[subtle]{savetrees}

\usepackage{subcaption}

\newtheorem{theorem}{Theorem}[section]
\newtheorem{corollary}{Corollary}[section]
\newtheorem{lemma}{Lemma}

\newcommand{\fattau}{\pmb{\tau}}

\theoremstyle{definition}
\newtheorem{definition}{Definition}

\title{Quantum Tensor Networks, Stochastic Processes, and Weighted Automata}

\author{%
  Siddarth Srinivasan\thanks{denotes equal contribution} \\
  School of Computer Science and Engineering\\
  University of Washington \\
  \texttt{sidsrini@cs.washington.edu} \\
  \And
  Sandesh Adhikary$^*$ \\
  School of Computer Science and Engineering\\
  University of Washington \\
  \texttt{adhikary@cs.washington.edu} \\
  \And
  Jacob Miller \\
 Mila \& DIRO \\ Universit\'{e} de Montr\'{e}al \\
  \texttt{jmjacobmiller@gmail.com} \\
  \And
 Guillaume Rabusseau\\
  CCAI chair - Mila \& DIRO\\ Université de Montréal\\
  \texttt{grabus@iro.umontreal.ca} \\
  \And
  Byron Boots \\
  School of Computer Science and Engineering \\
  University of Washington \\
  \texttt{bboots@cs.washington.edu}
}

\begin{document}
 
\maketitle
\begin{abstract}
Modeling joint probability distributions over sequences has been studied from many perspectives. The physics community developed \emph{matrix product states}, a tensor-train decomposition for probabilistic modeling, motivated by the need to tractably model many-body systems. But similar models have also been studied in the stochastic processes and weighted automata literature, with little work on how these bodies of work relate to each other. We address this gap by showing how stationary or uniform versions of popular quantum tensor network models have equivalent representations in the stochastic processes and weighted automata literature, in the limit of infinitely long sequences. We demonstrate several equivalence results between models used in these three communities: (i) uniform variants of matrix product states, Born machines and locally purified states from the quantum tensor networks literature, (ii) predictive state representations, hidden Markov models, norm-observable operator models and hidden quantum Markov models from the stochastic process literature, and (iii) stochastic weighted automata, probabilistic automata and quadratic automata from the formal languages literature.
Such connections may open the door for results and methods developed in one area to be applied in another.
\end{abstract}
\renewcommand{\thefootnote}{\arabic{footnote}}
\section{Introduction}

Matrix product states (MPS) were first developed by the physics community as compact representations of often intractable wave functions of complex quantum systems~\citep{perez2006matrix,Klumper1993MatrixPG,fannes1992finitely}, in parallel with the equivalent tensor-train decomposition~\citep{Oseledets2011TensorTrainD} developed in applied mathematics for high-order tensors. These tensor network models have been gaining popularity in machine learning especially as means of compressing highly-parameterized models \citep{novikov2015tensorizing,garipov2016ultimate,Yu2017LongtermFU,novikov2014putting}. There has also been recent interest in directly connecting ideas and methods from quantum tensor networks to machine learning \citep{stoudenmire2016supervised,han2018,guo2018matrix,huggins2019towards}. In particular, tensor networks have been used for probabilistic modeling as parameterizations of joint probability tensors \citep{glasser2019expressive,miller2020tensor,stokes2019probabilistic}. But the same problem has also been studied from various other perspectives. Notably, observable operator models \citep{jaeger2000observable} or predictive state representations (PSRs) \citep{SinghPSR} from the machine learning literature and stochastic weighted automata \citep{balle2014methods} are approaches to tackle essentially the same problem. While \citet{Thon2015} provide an overview discussing connections between PSRs and stochastic WA, their connection to MPS has not been extensively explored. At the same time, there exist many variants of tensor network models related to MPS that can be used for probabilistic modeling. \citet{glasser2019expressive} recently provided a thorough investigation of the relative expressiveness of various tensor networks for the \emph{non-uniform} case (where cores in the tensor decomposition need not be identical). However, to the best of our knowledge, similar relationships have  not yet been established for the \emph{uniform} case. We address these issues by examining how various quantum tensor networks relate to aforementioned work in different fields, and we derive a collection of results analyzing the relationships in expressiveness between uniform MPS and their various subclasses.  

The uniform case is important to examine for a number of reasons. The inherent weight sharing in uniform tensor networks leads to particularly compact models, especially when learning from highly structured data. This compactness becomes especially useful when we consider physical implementations of tensor network models in quantum circuits. For instance, \citet{glasser2019expressive} draw an equivalence between local quantum circuits and tensor networks; network parameters define gates that can be implemented on a quantum computer for probabilistic modeling. Uniform networks have fewer parameters, corresponding to a smaller set of quantum gates and greater ease of implementation on resource constrained near-term quantum computers. Despite the many useful properties of uniformity, the tensor-network literature tends to focus more on non-uniform models. We aim to fill this gap by developing expressiveness relationships for uniform variants.

We expect that the connections established in this paper will also open the door for results and methods in one area to be used in another. For instance, one of the proof strategies we adopt is to develop expressiveness relationships between subclasses of PSRs, and show how they also carry over to equivalent uniform tensor networks. Such cross fertilization also takes place at the level of algorithms. For instance, the learning algorithm for locally purified states (LPS) employed in \citet{glasser2019expressive} does not preserve uniformity of the model across time-steps, or enforce normalization constraints on learned operators. With the equivalence between uniform LPS and hidden quantum Markov models (HQMMs) established in this paper, the HQMM learning algorithm from \citet{adhikary2019expressiveness}, based on constrained optimization on the Stiefel manifold, can be adapted to learn uniform LPS \emph{while enforcing all appropriate constraints}. Similarly, spectral algorithms that have been developed for stochastic process models such as hidden Markov models (HMMs) and PSRs \citep{hsu2012spectral, siddiqi2010reduced, bailly2009grammatical} could also be adapted to learn uniform LPS and uniform MPS models. Spectral algorithms typically come with consistency guarantees, along with rigorous bounds on sample complexity. Such formal guarantees are less common in tensor network methods, such as variants of alternating least squares \citep{Oseledets2011TensorTrainD} or density matrix renormalization group methods \citep{dmrg}. On the other hand, tensor network algorithms tend to be better suited for very high-dimensional data; presenting an opportunity to adapt them to scale up algorithms for stochastic process models. 

Finally, one of our key motivations is to simply provide a means of translating between similar models developed in different fields. While prior works \citep{glasser2019expressive,kliesch2014matrix,critch2014algebraic} have noted similarities between tensor networks, stochastic processes and weighted automata, many formal and explicit connections are still lacking, especially in the context of model expressiveness. It is still difficult for practitioners in one field to verify that the model classes they have been working with are indeed used elsewhere, given the differences in nomenclature and domain of application; simply having a dictionary to rigorously translate between fields can be quite valuable. Such a dictionary is particularly timely given the growing popularity of tensor networks in machine learning. We hope that the connections developed in this paper will help bring together complementary advances occurring in these various fields.

\paragraph{Summary of Contributions}
In Section 2, we demonstrate that uniform Matrix states (uMPS) are equivalent to predictive state representations and stochastic weighted automata, when taken in the \emph{non-terminating limit} (where we evaluate probabilities sufficiently away from the end of a sequence). Section 3 presents the known equivalence between uMPS with non-negative parameters, HMMs, and probabilistic automata, to show in Section 4 that another subclass of uMPS called Born machines (BM) \citep{han2018} is equivalent to norm-observable operator models (NOOM) \citep{Zhao2010b} and quadratic weighted automata (QWA) \citep{bailly2011quadratic}. We also demonstrate that uBMs and NOOMs are relatively restrictive model classes in that there are HMMs with no equivalent finite-dimensional uBM or NOOM (HMMs $\nsubseteq$ NOOMs/uBMs). Finally, in Section 5, we analyze a broadly expressive subclass of uMPS known as locally purified states (LPS), demonstrate its equivalence to hidden quantum Markov models, and discuss the open question of how the expressiveness of uLPS relates to that of uMPS. We thus develop a unifying perspective on a wide range of models coming from disparate communities, providing a rigorous characterization of how they are related to one another, as illustrated in Figures ~\ref{fig:tensor_network_diagrams} and ~\ref{fig:model_subsets}. The proofs for all theorems are provided in the Appendix. In our presentation, we routinely point out connections between tensor networks and relevant concepts in physics. However, we note that these models are not restricted to this domain.

\paragraph{Notation}
We use bold-face for matrix and tensor operators (e.g. $\mathbf{A}$), arrows over symbols to denote vectors (e.g. $\vec{x}$), and plain non-bold symbols for scalars. The vector-arrows are also used to indicate vectorization (column-first convention) of matrices. We frequently make use of the ones matrix $\mathbf{1}$ (filled with $1$s) and the identity matrix $\mathds{I}$, as well as their vectorizations $\vec{\mathbf{1}}$ and $\vec{\mathds{I}}$. We use overhead bars to denote complex conjugates (e.g. $\bar{\mathbf{A}}$) and $\dagger$ for the conjugate transpose $(\bar{\mathbf{A}}^T = \mathbf{A}^\dagger)$. Finally, $\text{tr}(\cdot)$ denotes the trace operation applied to matrices, and $\otimes$ denotes the Kronecker product.

\section{Uniform Matrix Product States}

 Given a sequence of $N$ observations, where each outcome can take $d_i$ values, the joint probability of any particular sequence $y_1, \ldots, y_N$ can be written using the following tensor-train decomposition, which gives an MPS: 

\begin{align}
    P(y_1,\ldots,y_N)
    ~
    &\propto
    ~
    \text{MPS}_{y_1,\ldots,y_N} \nonumber\\
    &= \mathbf{A}^{[N],y_{N}} \mathbf{A}^{[N-1],y_{N-1}}
    ~
    \dots
    ~
    \mathbf{A}^{[2],y_{2}} \mathbf{A}^{[1],y_1}
    \label{eq_mps}
\end{align}

where each $\mathbf{A}^{[i]}$ is a three-mode tensor \emph{core} of the MPS containing $d_i$ slices, with the matrix slice associated with outcome $Y_i$ denoted by $\mathbf{A}^{[i], y_i}$. Each slice is a $D_{i+1} \times D_{i}$ matrix, and the conventional choice (which we use in this paper) of \emph{open boundary conditions}\footnote{An alternate choice, \emph{periodic boundary conditions}, sets $D_0=D_N$ and $D_0, D_N \geq 1$ and uses a trace operation to evaluate the product of matrices in Equation~\ref{eq_mps}. MPS with periodic boundaries are equivalent to the tensor ring decomposition~\citep{Mickelin2018TensorRD}.} is to set $D_0=D_N=1$ (i.e. $\mathbf{A}^{[1],y_1}$ and $\mathbf{A}^{[N],y_N}$ are column and row vectors respectively). MPS with open boundaries are equivalent to tensor train (TT) decompositions, and we will define them over the complex field, a choice common in quantum physics and tensor network settings.

The maximal value of $D = \max_k D_k$ is also called the bond-dimension or the TT-rank \citep{glasser2019expressive} of the MPS. 
\footnote{Operational characterizations of the bond dimension have been developed in quantum physics, in terms of entanglement~\citep{eisert2010} or the state space dimension of recurrent many-body dynamics which generate the associated wavefunction~\citep{Schoen2005SequentialGO}.} For fixed dynamics, this will lead the MPS cores $\mathbf{A}^{[i]}$ to be identical.

In this paper, we will focus on the ``uniform'' case of identical cores, i.e., a uniform MPS or uMPS. uMPS models were first developed in the quantum physics community~\citep{perez2006matrix,2016TangentSM}, although employing a different probabilistic correspondence (Born machines as discussed later) than described below. As we will discuss, this corresponds naturally to Markovian dynamical systems; the notion of \emph{past being independent of future given the present} is encoded by the tensor train structure where each core only has two neighbours. While an MPS is inherently defined with respect to a fixed sequence length, a uMPS can be applied to sequences of arbitrary fixed or infinite length~\citep{Cirac2010InfiniteMP}. As there should be no distinction between the cores at different time steps in a uMPS, a natural representation is to fix two boundary vectors $(\vec{\sigma},~\vec{\rho}_0)$, leading to the following decomposition of the joint probability:
\begin{align} 
    P(y_1,\dots,y_N) &= \text{uMPS}_{y_1,\dots,y_N}
    \nonumber \\
    &= \vec{\sigma}^\dagger \mathbf{A}^{y_N} \mathbf{A}^{y_{N-1}}
    \dots
    \mathbf{A}^{y_2} \mathbf{A}^{y_1}
    \vec{\rho}_0
    \label{eq_umps}
\end{align}

\begin{figure}
    \centering
    \includegraphics[scale=0.37]{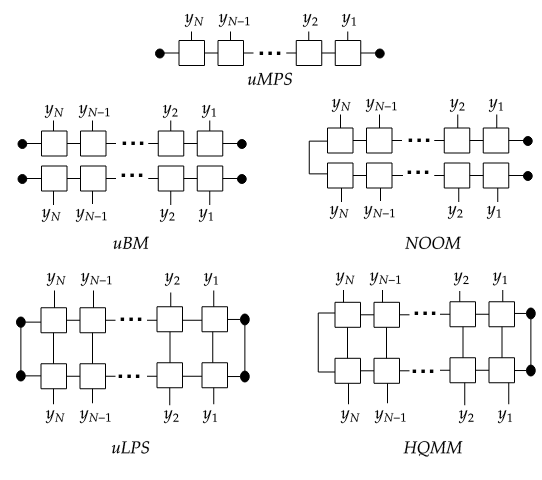}
    \caption{Tensor network diagrams}
    \label{fig:tensor_network_diagrams}
\end{figure}

To explore connections with arbitrary-length PSRs and WFAs, we will particularly focus on the \emph{non-terminating limit}. Consider that if we wished to compute the probability of some subsequence from $t= 1,\ldots,T$ of the $N$-length uMPS ($T < N$), we could compute $P(y_1, \ldots, y_T) = \sum_{y_N}\cdots\sum_{y_{T+1}}P(y_1,\ldots,y_T,y_{T+1},\ldots,y_N)$.  

The non-terminating limit is essentially when we consider the uMPS to be infinitely long, i.e., we compute the probability of the subsequence in the limit as $N \to \infty$.
\begin{definition}[Non-terminating uMPS]
A non-terminating uMPS is an infinitely long uMPS where we can compute the probability of any sequence $P(y_1,\ldots,y_T)$ of length $T$ by marginalizing over infinitely many future observations, i.e. $P(y_1,\ldots,y_T) = \lim_{N\to\infty}\sum_{y_N}\cdots\sum_{y_{T+1}}P(y_1,\ldots,y_T,y_{T+1},\ldots,y_N)$.
\end{definition}

This is a natural approach to modeling arbitrary length sequences with Markovian dynamics; intuitively, if given an identical set of tensor cores at each time step, the probability of a sequence should not depend on how far it is from the `end' of the sequence.  Similar notions are routinely used in machine learning and physics. In machine learning, it is common to discard the first few entries of sequences as ``burn-in'' to allow systems to reach their stationary distribution. In our case, the burn is being applied to the end of the sequence. The non-terminating limit is also similar to the ``thermodynamic limit'' employed in many-body physics, which marginalizes over an infinite number of future \emph{and} past observations \citep{2016TangentSM}. Such limits reflect the behavior seen in the interior of large systems, and avoid more complicated phenomena which arise near the beginning or end of sequences.

\subsection{The Many Names of Matrix Product States}
\label{sec_psr}
While connections between MPS and hidden Markov models (HMM) have been widely noted, we point out that non-terminating uMPS models have been studied from various perspectives, and are referred to by different names in the literature, such as stochastic weighted finite automata (stochastic WFA)~\citep{balle2014spectral}, quasi-realizations~\citep{vidyasagar2011complete}, observable operator models (OOM)~\citep{jaeger2000observable}, and (uncontrolled) predictive state representations (PSR)~\citep{SinghPSR}. The latter three models are exactly identical (we just refer to them as uncontrolled PSRs in this paper) and come from the stochastic processes perspective, while stochastic WFA are slightly different in their formulation, in that they are more similar to (terminating) uMPS (see below).  \citet{Thon2015} detail a general framework of \emph{sequential systems} to study how PSRs and WFA relate to each other. 

\begin{figure}
    \centering
    \includegraphics[scale=0.34]{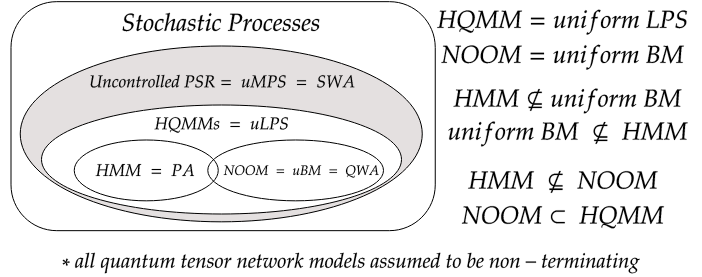}
    \caption{Subset relationships between stochastic process models, non-terminating uniform quantum tensor networks, and weighted automata, along with a summary of new relationships established in this paper. The grey area is potentially empty.}
    \label{fig:model_subsets}
\end{figure}

\paragraph{Predictive State Representations} We write the stochastic process defined by an $n$--dimensional predictive state representation over a set of discrete observations $\mathcal{O}$ as a tuple $(\mathds{C}^n, \vec{\sigma}, \{{\bf \pmb{\tau}}_y\}_{y\in\mathcal{O}}, \vec{x}_0)$. The initial state $\vec{x}_0 \in \mathds{C}^n$ is normalized, as enforced by the linear evaluation functional $\vec{\sigma}$, i.e., $\vec{\sigma}^\dagger \vec{x}_0 = 1$, and the observable operators are constrained to have normalized marginals over observations  $\vec{\sigma}^\dagger\sum_y \pmb{\tau}_y = \vec{\sigma}^\dagger$, i.e., $\vec{\sigma}^\dagger$ is a fixed point of the \emph{transfer operator} $\sum_y {\pmb{\tau}}_y$. The probability of arbitrary length sequences $y_1, \ldots, y_T \in \mathcal{O}^{T}$ is computed as $\vec{\sigma}^\dagger\pmb{\tau}_{y_T} \ldots \pmb{\tau}_{y_1}\vec{x}_0$, which should be non-negative for any sequence. Note that we simply require this to hold for a valid PSR; we do not explicitly enforce constraints to ensure this. This joint probability computation is identical to Equation~\ref{eq_umps}, where  evaluation functional $\vec{\sigma}$ and the initial state $\vec{x}_0$  are analogous to the left and right boundary vectors of the uMPS, and the \emph{observable operators} $\pmb{\tau}_y$ correspond to the matrix slices $\mathbf{A}^{y}$. In this sense, both uMPS and PSRs define tensor-train decompositions of joint distributions for a given fixed number of observations $T$. The only difference is that a uMPS does not require its evaluation functional to be the fixed point of its transfer operator. However, as we now discuss, any arbitrary uMPS evaluation functional will eventually converge to the fixed point of its transfer operator in the non-terminating limit. The fixed point then becomes the \emph{effective} evaluation functional of the uMPS in this limit.

Since PSRs were formulated with dynamical systems in mind, we typically consider sequences of \emph{arbitrary} length, whose probabilities are determined via a hidden state which evolves under a time-invariant update rule: the state update conditioned on an observation $y_{t}$ is computed as $\vec{x}_t = \frac{\pmb{\tau}_{y_t}\vec{x}_{t-1}}{\vec{\sigma}^T\pmb{\tau}_{y_t}\vec{x}_{t-1}}$ and the probability of an observation $y_{t}$ is $P(y_t|\vec{x}_t) = \vec{\sigma}^T\pmb{\tau}_{y_t}\vec{x}_{t}$. This allows us to deal more flexibly with arbitrary length sequences as compared to a generic uMPS. This flexibility for arbitrary-length sequences is precisely why we consider non-terminating uMPS: we can compute the conditional probability of a sequence $P(y_{t}|y_{1:t-1})$ by marginalizing over all possible future observations with a relatively simple equation:

\begin{small}
\begin{align}\label{eq:condprob}
P(y_{t}|y_{1:t-1}) &= \frac{\sum_{y_N, \ldots, y_{t+1}} P(y_N, \ldots, y_{t+1}, y_t, y_{t-1}, \ldots y_1)}{\sum_{y_N, \ldots, y_{t}} P(y_N, \ldots, y_t, y_{t-1}, \ldots, y_1)} \nonumber \\
&= \frac{\vec{\sigma}^\dagger \left(\sum_y \pmb{\tau}_y\right)^{N-(t+1)} \pmb{\tau}_{y_{t}} \dots \pmb{\tau}_1 \vec{\rho}_0}{\vec{\sigma}^\dagger \left(\sum_y \pmb{\tau}_{y}\right)^{N-t} \pmb{\tau}_{y_{t-1}} \dots \pmb{\tau}_1 \vec{\rho}_0} 
\end{align}
\end{small}

Thus the \emph{effective} evaluation functional $\vec{\sigma}^\dagger \left(\sum_y \pmb{\tau}_y\right)^{N-t}$ is a function of time and so different at every time step in general. However, if the transfer operator $\fattau = \sum_y \pmb{\tau}_y$ has some fixed point $\vec{\sigma}_*^\dagger$, i.e., $\vec{\sigma}_*^\dagger \fattau =  \vec{\sigma}_*^\dagger$, then the effective evaluation functional at timestep $t$ (which is $N-t$ steps away from the left boundary of the uMPS) will eventually converge to $\vec{\sigma}_*$ given a long enough sequence.\footnote{In fact, the effective evaluation functional will converge at an exponential rate towards the fixed point, so that $\| \vec{\sigma}_t - \vec{\sigma}_*\|^2 = \mathcal{O}(\exp \tfrac{N-t}{\xi})$, with a ``correlation length'' $\xi \simeq (1 - |\lambda_2| / |\lambda_1|)^{-1}$ set by the ratio of the largest and second largest eigenvalues of the transfer operator~\citep{orus2014}. Transfer operators with non-degenerate spectra can always be rescaled to have a unique fixed point, while matrices with degenerate spectra form a measure zero subset.} So, as long as we remain sufficiently far from the end of a sequence, the particular choice of the the left boundary vector does not matter.

Given that the non-terminating limit effectively replaces the uMPS evaluation functional $\vec{\sigma}$ by the fixed point $\vec{\sigma_*}$, consider what happens if we require $\vec{\sigma} = \vec{\sigma_*}$ to begin with, as is the case for PSRs. In this case, our effective evaluation functional remains independent of $t$, permitting a simple recursive state update rule that does not require fixing a prior sequence length or marginalizing over future observations. In this sense, a non-terminating uMPS is strictly equivalent to a PSR, a relationship which we will see holds between several other model families.

\begin{theorem}
Non-terminating uniform matrix product states are equivalent to uncontrolled predictive state representations.
\end{theorem}

If we do not consider the non-terminating limit of a uMPS, the subsequence length and boundary choice will affect the probability computed. Then, we technically only have an equivalence with PSRs for a fixed sequence length (when the evaluation functional is a fixed point of the transfer operator) and no notion of recursive state update.

\paragraph{Stochastic Weighted Automata} A weighted automaton~(WA) is a tuple $(\mathds{C}^n, \vec{\sigma}, \{{\bf \pmb{\tau}}_y\}_{y\in\mathcal{O}}, \vec{x}_0)$ which computes a function $f(y_1, \ldots, y_T) = \vec{\sigma}^\dagger\pmb{\tau}_{y_T} \ldots \pmb{\tau}_{y_1}\vec{x}_0$.  In contrast with PSRs, no constraints are enforced on the weights of a weighted automaton in general.  A weighted automaton is \emph{stochastic} if it computes a probability distribution. These models constitute another class of models equivalent to PSRs and uMPS, and represent probability distributions over sequences of symbols from an alphabet $\mathcal{O}$. It is worth mentioning that the semantics of the probabilities computed by PSRs and stochastic WAs can differ: while PSR typically maintain a recursive state and are used to compute the probability of a given sequence conditioned on some past sequence, stochastic WA are often used to compute the joint distributions over the set of all possible finite length sequences (just as in uMPS). We refer the reader to \citet{Thon2015} for a unifying perspective.

\section{Non-Negative uMPS, Hidden Markov Models, and Probabilistic Automata}
\label{sec_mps_hmm}

We first point out a well-known connection between hidden Markov models (HMM) and matrix product states~\citep{kliesch2014matrix, Critch2014AlgebraicGO}. We refer to uMPS where all tensor cores and boundary vectors are non-negative as \emph{non-negative uMPS}. HMMs have been extensively studied in machine learning and are a common approach to modeling discrete-observation sequences where an unobserved hidden state undergoes Markovian evolution (where the future is independent of past given present) and emits observations at each time-step \citep{Rabiner1986AnIT}. HMMs can be thought of as a special case of PSRs where all parameters are constrained to be non-negative. Such models are usually characterized by an initial \emph{belief state} $\vec{x}_0$ and column-stochastic transition and emission matrices ${\bf A}$ and ${\bf C}$. 
 Formally, we give the following definition:
 \begin{definition}[Hidden Markov Model]\label{def_hmm}
An $n-$dimensional Hidden Markov Model for a set of discrete observations $\mathcal{O}$ is a stochastic process described by the tuple $(\mathds{R}^n, {\bf A}, {\bf C}, \vec{x}_0)$. The transition matrix ${\bf A}\in \mathds{R}^{n\times n}_{\geq 0}$ and the emission matrix ${\bf C} \in \mathds{R}^{|\mathcal{O}| \times n}_{\geq 0}$ are non-negative and column stochastic i.e. $\vec{{\bf 1}}^T {\bf A} = \vec{{\bf 1}}^T {\bf C} = \vec{{\bf 1}}^T$. The initial state $\vec{x}_0 \in \mathds{R}^{n}_{\geq 0}$ is also non-negative and is normalized $||\vec{x}_0||_1 = \vec{{\bf 1}}^T \vec{x} = 1$. 
\end{definition}
The state transitions through the simple linear update $\vec{x}_t' = \mathbf{A} \vec{x}_{t-1}$. To condition on observation $y$, we construct the diagonal matrix $\text{diag}(\mathbf{C}_{(y,:)})$ from the $y^{th}$ row of $\mathbf{C}$, and perform a normalized update $
\vec{x}_t | y_t = \frac{\text{diag}(\mathbf{C}_{(y,:)}) \vec{x}_t'}{\vec{\mathbf{1}}^T \text{diag}(\mathbf{C}_{(y,:)})\vec{x}_t'}$. This multi-step filtering process can be simplified using an alternative representation with \emph{observable operators} as ${\bf T}_y = \text{diag}({\bf C}_{(y,:)}){\bf A}$, where we rewrite the normalization constraints on operators as ${\bf 1}^T \sum_y {\bf T}_y = {\bf 1}^T$. Then, we recover a recursive state update $\vec{x}_t = \frac{{\bf T}_{y_{t-1}}\vec{x}_{t-1}}{\vec{\mathbf{1}}^T{\bf T}_{y_{t-1}}\vec{x}_{t-1}}$. Clearly, HMMs are a special case of PSRs where the parameters are restricted to be non-negative, and a special case of uMPS when the left boundary $\vec{\sigma}^\dagger = {\bf 1}^T$, the right boundary $\vec{\rho}_0 =\vec{x}_0$, and the tensor core slice $A^y = {\bf T}_y$.

\paragraph{Probabilistic Automata} Lastly, non-negative uMPS are equivalent to  probabilistic automata from formal language theory~\cite[Section 4.2]{denis2008rational}, which are in essence weighted automata where transition matrices need to satisfy stochasticity constraints. The strict equivalence between probabilistic automata and HMMs is proved in~\citet[Proposition 8]{dupont2005links} (see also Section 2.2 in~\citet{balle2014methods}). In addition, it is known that non-negative uMPS are strictly less expressive than general uMPS for representing probability distributions; a proof of this result in the context of formal languages can be found in~\citet{denis2008rational}.  We give a brief discussion of this next.

\subsection{The Negative Probability Problem and the Expressiveness of Finite PSRs}\label{sec_npp}

As noted by several authors \citep{jaeger2000observable, adhikary2019expressiveness}, PSRs lack a constructive definition; 
the definition of a PSR simply demands that the probabilities produced by the model be non-negative without describing constraints that can achieve this. Indeed, this is the cost of relaxing the non-negativity constraint on HMMs; it is undecidable whether a given set of PSR parameters will assign a negative probability to some arbitrary-length sequence \citep{denis2004learning, wiewiora2008modeling}, an issue known as the \emph{negative probability problem} (NPP). A similar issue arises in the many-body physics setting, where the analogous question of whether a general matrix product operator describes a non-negative quantum density operator is also undecidable~\citep{kliesch2014matrix}. In the special case where all PSR parameters are non-negative, we have a sufficient condition for generating valid probabilities, namely that the PSR is a hidden Markov model. Otherwise, the best available approach to characterize valid states for PSRs is whether they define a pointed convex cone (that includes the initial state) which is closed under its operators, and all points in it generate valid probabilities \citep{heller1965stochastic, jaeger2000observable, adhikary2019expressiveness}. 

While this undecidability is an inconvenient feature of PSRs, it turns out that constraining PSRs to have only non-negative entries comes with a reduction in expressive power; there are finite (bond/state) dimensional uMPS/PSRs which have no equivalent finite-dimensional HMM representations for arbitrary length sequences (for example, the probability clock in \citet{jaeger2000observable}). The general question of which uncontrolled PSRs have equivalent finite-dimensional HMMs (though not always discussed in those terms) is referred to by some as the \emph{positive realization problem} \citep{benvenuti2004tutorial, vidyasagar2011complete}. A common approach is to use the result that a PSR has an equivalent finite-dimensional HMM if and only if the aforementioned convex cone of valid initial states $\{\vec{x}_0\}$ for a set of given operators $\vec{\sigma}^\dagger$, $\{\fattau_y\}$ is $k$-polyhedral for some finite $k$ \citep{jaeger2000observable}. 

There has been some work trying to investigate whether it is possible to maintain the superior expressiveness of uMPS/PSRs while avoiding the undecidability issue. \citet{Zhao2010b, bailly2011quadratic,  adhikary2019expressiveness} explore this question in the machine learning context, while \citet{glasser2019expressive} consider this problem from the quantum tensor network perspective. We will explore these proposals shortly. When discussing the relative expressiveness of a model compared to a uMPS/PSR, 
if its bond dimension (i.e. state dimension) grows with sequence length, we say there is no equivalent parameterization of the uMPS/PSR distribution in this model class.

\section{Uniform Born Machines, Norm-Observable Operator Models, and Quadratic Weighted Automata}
\label{sec_noom_ubm}

Born machines (BMs)~\citep{han2018} are a popular class of quantum tensor networks that model probability densities as the absolute-square of the outputs of a tensor-train decomposition, and hence always output valid probabilities. 
As with uMPS, we will work with uniform Born machines (uBMs)~\citep{miller2020tensor}, for which the joint probability of $N$ discrete random variables $\{Y_i\}_{i=1}^N$ is computed as follows (with boundary vectors $\vec{\alpha}$ and $\vec{\omega}_0$ sandwiching a sequence of identical cores ${\bf A}$):
\begin{align}
    P(y_1, \ldots, y_N) = \text{uBM}_{y_1, \ldots, y_N} = \left|\vec{\alpha}^{~\dagger} \mathbf{A}^{y_N} \dots \mathbf{A}^{y_1} \vec{\omega}_0\right|^2
    \label{eq_bms}
\end{align}
We can re-write this decomposition showing uBMs to be special kinds of uMPS/PSR:
\begin{small}
\begin{equation*}
\begin{split}\label{eq_bmj}
    \text{uBM}_{y_1,\dots,y_N}
    ~~&=~~
    \left|\vec{\alpha}^{~\dagger} \mathbf{A}^{y_N} \dots \mathbf{A}^{y_{1}} \vec{\omega}_0  \right|^2 \\
    ~~&=~~
    \vec{\alpha}^{~\dagger} \mathbf{A}^{y_N}\dots \mathbf{A}^{y_{1}} \vec{\omega}_0~\vec{\omega}_0^\dagger (\mathbf{A}^{y_{1}})^\dagger \dots (\mathbf{A}^{y_N})^\dagger  \vec{\alpha}\\ 
    ~~&=~~
    \vec{\sigma}^{~\dagger}~\fattau_{y_N}~\dots~\fattau_{y_1}~\vec{\rho}_0
    \end{split}
\end{equation*}
\end{small}%
where $\fattau_{y} = \overline{\mathbf{A}^{y}} \otimes {\mathbf{A}^{y}}$, $\vec{\rho}_0 = \overline{\vec{\omega}_0} \otimes \vec{\omega}_0$, and $\vec{\sigma} = \overline{\vec{\alpha}} \otimes \vec{\alpha}$. This makes it clear that uBMs are a special class of MPS/PSRs, where the observable operators  $\fattau_y$ and boundary conditions must satisfy the additional requirement of having unit Kraus-rank (i.e., a symmetric unit Schmidt rank decomposition).\footnote{An operator ${\bf A}$ has unit Kraus-rank if it has a decomposition ${\bf A} = \overline{\bf X} \otimes {\bf X}$.} 

\paragraph{Norm Observable Operator Models}
Motivated by the NPP for PSRs, \citet{Zhao2010b} introduce norm-observable operator models or NOOMs. Coming from the PSR literature, they were designed to model joint distributions of observations as well as recursive state-updates to obtain conditional probabilities (analogous to PSRs in Section \ref{sec_psr}).  They bear a striking resemblance to uniform Born machines (uBMs) and the connection has not been previously explored. Both NOOMs and uBMs associate probabilities with quadratic functions of the state vector, with NOOMs directly using the squared 2-norm of the state to determine observation probabilities. While NOOMs were originally defined on the reals, we use a more general definition over complex numbers.
\begin{definition}[Norm-observable operator model]
An $n$-dimensional norm-observable operator model for a set of discrete observations $\mathcal{O}$ is a stochastic process described by the tuple $(\mathds{C}^n, \{\pmb{\phi}_y\}_{y \in \mathcal{O}}, \vec{\psi}_0)$. The initial state $\vec{\psi}_0 \in \mathds{C}^n$ is normalized by having unit $2$-norm i.e. $\|\vec{\psi}_0\|_2^2 = 1$. The operators $\pmb{\phi}_y \in \mathds{C}^{n \times n}$ satisfy $\sum_{y} \pmb{\phi}_y^\dagger \pmb{\phi}_y = \mathds{I}$.
\end{definition}
 These models avoid the NPP by using the 2-norm of the state to recover probability which, unlike for HMMs, is insensitive to the use of negative parameters in the matrices $\pmb{\phi}_{y}$. We write the joint probability of a sequence as computed by a NOOM and manipulate it using using the relationship between 2-norm and trace to show:
\begin{equation}
\label{eq_noom}
\begin{split}
    P(y_1, \dots, y_N) 
    &= \text{NOOM}_{y_1,\dots,y_N} \\
    &= \left|\left|
    \pmb{\phi}_{y_N} \dots \pmb{\phi}_{y_{1}} \vec{\psi}_0  \right|\right|_2^2 \\ &= \text{tr}(\pmb{\phi}_{y_N} \dots \pmb{\phi}_{y_{1}} \vec{\psi}_0 \vec{\psi}_0^\dagger (\pmb{\phi}_{y_{1}} )^\dagger \dots (\pmb{\phi}_{y_N})^\dagger) \\
    &= \vec{\mathds{I}}^\dagger\, \fattau_{y_N}~\dots~\fattau_{y_{1}}~\vec{\rho}_0
    \end{split}
\end{equation}
where $\fattau_{y} = \overline{\pmb{\phi}_{y}}\otimes \pmb{\phi}_{y} \in \mathds{C}^{n^2 \times n^2}$ and $\vec{\rho}_0 = \overline{\vec{\psi}_0}\otimes \vec{\psi}_0$. Equation~\ref{eq_noom} shows that NOOMs are a special subset of PSRs/MPS, as every finite-dimensional NOOM has an equivalent finite-dimensional PSR $(\mathds{C}^{n^2}, \{\overline{\pmb{\phi}_y} \otimes \pmb{\phi}_y\}_{y \in \mathcal{O}}, \overline{\vec{\psi}_0} \otimes \vec{\psi}_0)$ \citep{Zhao2010b}. From a quantum mechanical perspective, the unit rank constraint on NOOM initial states can be framed as requiring the initial state to be a pure density matrix.
We can also recursively update the state conditioned on observation $y_t$ as $\vec{\psi}_t = \frac{\pmb{\phi}_{y_t}\vec{\psi}_{t-1}}{\|\pmb{\phi}_{y_t}\vec{\psi}_{t-1}\|_2}$, where $y_t$ is observed with probability $P(y_t|\vec{\psi}_t) = \|\pmb{\phi}_{y_t} \vec{\psi}_{t-1}\|_2^2$. 

\paragraph{Non-terminating uniform BMs are NOOMs} Note that the NOOM joint distribution in Equation~\ref{eq_noom} is almost identical to that of uBMs in Equation \ref{eq_bmj}, with $\pmb{\tau}_{y}$ and $\vec{\rho}_0$ having unit Kraus rank; but the left boundary / evaluation functional $\vec{\sigma} = \vec{\alpha} \otimes \vec{\alpha}$ is replaced by $\vec{\sigma} = \vec{\mathds{I}}$ and necessarily is full Kraus rank. So how can we reconcile these nearly identical models? Similar to our approach in Section~\ref{sec_psr}, we can consider the uBM for an infinitely long sequence where the exact specification of the left boundary / evaluation functional ceases to matter in the non-terminating limit; the effective evaluation functional converges to the fixed point of the transfer operator and we have a notion of recursive state update. Assuming that the uBM transfer operator is a similarity transform away from a trace-preserving quantum channel (i.e., which is normalized by satisfying $\sum_y \pmb{\phi}_y^\dagger \pmb{\phi}_y = \mathds{I}$; see appendix for more details), we have that an arbitrary evaluation functional (with unit Kraus-rank) of such a uBM will eventually converge to $\vec{\mathds{I}}$, the NOOM's evaluation functional:
\begin{theorem}
Non-terminating uniform Born machines are equivalent to norm observable operator models.
\end{theorem}

With the above equivalence, we now turn to the question of how the expressiveness of uBM/NOOMs compares to non-negative uMPS/HMMs,  As we have seen, they are all special classes of uMPS, but with different constructions. \citet{glasser2019expressive} studied the expressiveness of \emph{non-uniform} BMs, showing that there are finite-dimensional non-uniform BMs that cannot be modeled by finite dimensional non-uniform HMMs, and conjecture that the reverse direction is also true. \citet{Zhao2010b} showed by example the existence of a NOOM (and so a non-terminating uBM) that cannot be modeled by any finite-dimensional HMM. However, they left open the question of whether HMMs were a subclass of NOOMs. We answer this question in the following theorem, which also implies the latter corollary through its equivalence with non-terminating uBM.

\begin{theorem}[HMM $\nsubseteq$ NOOM]
\label{thm_hmm_noom}
There exist finite-dimensional hidden Markov models that have no equivalent finite-dimensional norm-observable operator model.
\end{theorem}

\begin{corollary}[uBM $\nsubseteq$ HMM and HMM $\nsubseteq$ uBM]
There exist finite-dimensional non-terminating uniform Born machines that have no equivalent finite-dimensional hidden Markov models, and vice-versa.
\end{corollary}

\paragraph{Quadratic Weighted Automata} Finally, we note that Quadratic Weighted Automata (QWA) \citep{bailly2011quadratic}, developed in the stochastic weighted automata literature, are  equivalent to uBM. \citet{bailly2011quadratic} suggest that QWA $\nsubseteq$ HMM and that HMM $\nsubseteq$ QWA, but do not provide a proof. To the best of our knowledge, the proof we provide is the first to formally show the non-equivalence of QWA and HMM. 

\section{Locally Purified States and Hidden Quantum Markov Models}
\label{sec_hqmms}
While uBMs/NOOMs are constructive models guaranteed to return valid probabilities, they still aren't expressive enough to capture all HMMs, a fairly general class. Hence, it may be desirable to identify a construction that is more expressive than these models but still gives valid probabilities. Locally Purified States (LPS) were proposed as a tensor-network model of discrete multivariate probability distributions inspired from techniques used in the simulation of quantum systems. \citet{glasser2019expressive} show that these models are not only strictly more expressive than non-uniform HMMs, but also correspond directly to local quantum circuits with ancillary qubits -- serving as a guide to design quantum circuits for probabilistic modeling. We arrive at the LPS model from the MPS model essentially by marginalizing over an additional mode -- called the ``purification dimension'' -- in each of the MPS tensors.  The rank of an LPS, also called its puri-rank, is defined the same way as the bond dimension (or TT-rank) for the MPS. The corresponding uniform LPS defines the unnormalized probability mass function over $N$ discrete random variables $\{Y_i\}_{i=1}^N$ as follows:
\begin{equation}
    \begin{split}
    \label{eq_joint_lps}
    &P(y_1, \ldots, y_N) = \text{uLPS}_{y_1, \ldots, y_N}\\ 
    &= \left( \sum_{\beta_L} \overline{\mathbf{K}}_{\beta_L, L}^T \otimes \mathbf{K}_{\beta_L, L}^T \right) \left( \sum_{\beta} \overline{\mathbf{K}}_{\beta, y_N} \otimes \mathbf{K}_{\beta, y_N} \right) \cdots \\
    &\cdots \left( \sum_{\beta} \overline{\mathbf{K}}_{\beta, y_1} \otimes \mathbf{K}_{\beta, y_1} \right) \left( \sum_{\beta_R} \overline{\mathbf{K}}_{\beta_R, R} \otimes \mathbf{K}_{\beta_R, R} \right)
    \end{split}
\end{equation} 

\paragraph{Hidden Quantum Markov Models}
Hidden quantum Markov models (HQMMs) were developed by \citet{monras2010hidden} as a quantum generalization of hidden Markov models that can model joint probabilities of sequences and also allow for recursive state updates we have described previously. \citet{srinivasan2017learning, srinivasan2018learning} specifically develop HQMMs by constructing quantum analogues of classical operations on graphical models, and show that HQMMs are a more general model class compared to HMMs. \citet{adhikary2019expressiveness} on the other hand develop HQMMs by relaxing the unit Kraus-rank constraint on NOOM operators and initial state. We give a formal definition of these models here (noting that the Choi matrix ${\bf C}_y$ is a particular reshuffling of the sum of superoperators ${\bf L}_y$ defined below, see \citet{adhikary2019expressiveness}):

\begin{definition}[Hidden Quantum Markov Models]
\label{def:hqmm}
An $n^2-$dimensional hidden quantum Markov model for a set of discrete observations $\mathcal{O}$ is a stochastic process described by the tuple $(\mathds{C}^{n^2}, \vec{\mathds{I}}, {\{\bf L}_y\}_{y\in \mathcal{O}}, \vec{\rho})$. The initial state $\vec{\rho}_0 \in \mathds{C}^{n^2}$ is a vectorized unit-trace Hermitian PSD matrix of arbitrary rank, so $\vec{\mathds{I}}^{~T} \vec{\rho}_0=1$. The Liouville operators ${\bf L}_y \in \mathds{C}^{n^2\times n^2}$ (with corresponding Choi matrices ${\bf C}_y$) are trace-preserving (TP) i.e. $\vec{\mathds{I}}^T \left(\sum_y {\bf L}_y \right) = \vec{\mathds{I}}^T$, and completely positive (CP) i.e. ${\bf C}_y \geq 0$.
\end{definition}

The CP-TP condition on the operator ${\bf L}_y$ implies that we can equivalently write it via the Kraus decomposition as ${\bf L}_y = \left( \sum_{\beta} \mathbf{K}_{\beta, y} \otimes \overline{\mathbf{K}}_{\beta, y} \right)$, using Kraus operators $\mathbf{K}_{\beta, y}$ \citep{Kraus1971, adhikary2019expressiveness}. Intuitively, what makes HQMMs more general than NOOMs is that its state can be a vectorized density matrix of arbitrary rank and the superoperators can have arbitrary Kraus-rank, while NOOMs require both these ranks to be 1. With this in mind, we can write and manipulate the joint probability of a sequence of $N$ observations as:
\begin{equation}
\label{eq_hqmm_lps}
\begin{split}
    &P(y_1, \ldots, y_N) = \text{HQMM}_{y_1, \ldots, y_N} = \vec{\mathds{I}}^T {\bf L}_{y_1} \cdots {\bf L}_{y_N}  \vec{\rho}_0  \\
    &= \vec{\mathds{I}}^T \left( \sum_{\beta} \mathbf{K}_{\beta, y_N} \otimes \overline{\mathbf{K}}_{\beta, y_N} \right)  \left( \sum_{\beta} \overline{\mathbf{K}}_{\beta, Y_1} \otimes \mathbf{K}_{\beta, y_1} \right) \vec{\rho}_0
    \end{split}
\end{equation}
The joint probability computation makes it clear that HQMMs are a class of PSRs, and the manipulation shows how they are equivalent to a uLPS where the left boundary condition is $\vec{\mathds{I}}$. We also compute the recursive state update conditioned on observation $y$ as  $\vec{\rho}_{t+1} = \frac{{\bf L}_y\vec{\rho}_t}{\vec{\mathds{I}}^T{\bf L}_y\vec{\rho}_t}$ and 
the probability of an observation $y$ is $P(y|\rho_t) = \vec{\mathds{I}}^T {\bf L}_y \vec{\rho}$.

\paragraph{Non-terminating Uniform LPS are HQMMs}
Equation \ref{eq_hqmm_lps} shows that every HQMM is a uLPS, but we also consider in what sense every uLPS is an HQMM: the transfer operator of arbitrary CP maps with unit spectral radius\footnote{This condition is necessary for probability distributions such as Equations~\ref{eq_bmj} and \ref{eq_joint_lps} to be properly normalized.} is a similarity transform away from that of a CP-TP map \citep{perez2006matrix}, so $\vec{\mathds{I}}$ is related to such a fixed point by such a similarity transform. Thus, every non-terminating uLPS has an equivalent HQMM and allows for an HQMM-style recursive state update. This is the same reasoning behind the equivalence between non-terminating uBMs (with CP maps) and NOOMs (with CP-TP maps).
\begin{theorem}
Non-terminating uniform locally purified states are equivalent to hidden quantum Markov models.
\end{theorem}

While it is already known that HMMs are a strict subset of HQMMs (since HQMMs also contain NOOMs which cannot always be modeled by a HMM), \citet{adhikary2019expressiveness} left open the possibility that every HQMM could have an equivalent NOOM in with some higher dimensional state. In light of Theorem \ref{thm_hmm_noom}, we can say this is not possible as NOOMs do not capture HMMs, while HQMMs can. We defer a longer discussion of how the expressiveness of HQMMs/uLPS compares to uMPS to the appendix (see Appendix~\ref{app_hqmm_express}), with the simple remark that we are not currently aware of an HQMM without an equivalent uMPS, although we may be able to adapt an example from the LPS $\subset$ MPS result from the non-uniform case \citep{glasser2019expressive}.
\begin{corollary}[NOOM $\subset$ HQMM]
\label{corollary_hqmm_noom}
Finite dimensional norm-observable operator models are a strict subset of finite dimensional hidden quantum Markov models.
\end{corollary}

We are not aware of any proposals from  the weighted automata literature that are analogous to these uLPS/HQMMs.

\paragraph{Expressiveness of HQMMs (uLPS) and PSRs (uMPS)} We have determined that HQMMs are the most constructive known subclass of PSRs (containing both NOOMs and HMMs), yet the question of whether there is a `gap' between HQMMs and PSRs, i.e., if there is a PSR which has no finite-dimensional HQMM representation, is still open to the best of our knowledge. The results in \cite{glasser2019expressive} and \citet{de2013purifications} show that MPS are more expressive than LPS in the \emph{non-uniform} case, but their technique cannot be easily adapted to the uniform case. We are not aware of an example of a PSR with no equivalent finite-dimensional HQMM. A longer discussion of this problem is presented in Appendix \ref{app_hqmm_express}.
\section{Conclusion}
We presented uniform matrix product states and their various subclasses, and showed how they relate to previous work in the stochastic processes and weighted automata literature.  In discussing the relative expressiveness of various models, we discuss if we can find an \emph{equivalent finite-dimensional parameterization} in another model class, but we do not discuss the relative compactness of various parameterizations. \citet{glasser2019expressive} do discuss this for the non-uniform case, and this could be an interesting direction to explore for the uniform case. We also speculate that the connections laid out here may allow spectral learning algorithms commonly used for PSRs and weighted automata suitable \citep{hsu2012spectral, balle2014spectral, hefny2015supervised} for learning uMPS, and an algorithm for constrained optimization on the Stiefel manifold \citep{adhikary2019expressiveness} suitable for learning uLPS \emph{with appropriate constraints}. Future work will involve adapting these algorithms so they can be transferred between the two fields.

We can extend our analyses to \emph{controlled} stochastic processes. Controlled generalizations of uMPS may be developed through matrix product operators \citep{Murg2008MatrixPO, Chan2016MatrixPO} that append an additional open index at each core of a uMPS, which we can associate with actions. We can also develop input-output versions of uniform tensor networks and uncontrolled stochastic process models, similar to input-output OOMs from \citet{Jaeger2003DiscretetimeDO}. We briefly describe such extensions for HQMMs and uLPSs in Appendix~\ref{app_sec_controlled_stoc_proc}, showing that they generalize recently proposed quantum versions of partially observable Markov decisions processes \citep{qomdp, YingYing, Cidre2016}. With this connection, we also find that the undecidability of perfect planning (determining if there exists a policy that can deterministically reach a goal state from an arbitrary initial state in finite steps) established for quantum POMDPs by \citet{qomdp} also extends to these generalizations. We leave a longer discussion for future work.

\newpage 

\bibliographystyle{apalike}
\bibliography{biblio.bib}

\newpage
\begin{appendix}

\section{Proofs of Theorems}

\begin{customthm}{2.1}
\label{app_thm_umps_psr}
Non-terminating uniform matrix product states are equivalent to uncontrolled predictive state representations.
\end{customthm}
\begin{proof}
Consider a uMPS $\{\vec{\sigma}, \{\fattau_y\}_y, \vec{\rho}_0\}$ defined over sequences of length $N$ and define the transfer operator $\fattau := \sum_y \fattau_y$. Say we want to compute the probability of observing $y_t$ at time $t$, conditioned on past observations $y_{1:t-1}$. This requires us to not only condition on the past observations, but also marginalize over all possible future sequences $y_{t:N}$ (for example, see Theorem $1$ in \citet{miller2020tensor} for the case of uBMs). The probability is calculated as follows
\begin{equation}
    P(y_t~|~y_{1:t}) = \frac{\vec{\sigma}^\dagger\fattau^{N-t} \fattau_{y_t} \left(\fattau_{y_{t-1}}\cdots \fattau_{y_1}\vec{\rho}_0\right)}
    {
    \vec{\sigma}^\dagger\fattau^{N-t} \fattau \left(\fattau_{y_{t-1}}\cdots \fattau_{y_1}\vec{\rho}_0\right)
    }
    = 
    \frac{\vec{\sigma}_t^\dagger \fattau_{y_t} \left(\fattau_{y_{t-1}} \cdots \fattau_{y_1}\vec{\rho}_0\right)}
    {
    \vec{\sigma}_t^\dagger \fattau \left(\fattau_{y_{t-1}} \cdots \fattau_{y_1} \vec{\rho}_0\right)
    }
    \label{app_eq_umps_prob}
\end{equation}
Here, we have defined the \emph{effective} evaluation functional $\vec{\sigma}_t = \left(\fattau^{N-t}\right)^\dagger \vec{\sigma}$ at time $t$. Intuitively, this represents the evaluation functional after having marginalized over all possibilities for the remaining $N-t$ time steps. We perform this marginalization via the transfer operator $\fattau^\dagger$. As $N \to \infty$ for fixed $t$, the trajectory of the effective evaluation functional will be strongly determined by the spectral properties of $\fattau$. Here, we restrict ourselves to transfer operators where the magnitudes of the two largest eigenvalues are distinct. Note that this is not a strong requirement, given that matrices with degenerate spectra form a measure zero subset of general square matrices.\footnote{uMPS which violate this non-degeneracy condition are associated with ``(anti-)ferromagnetic order'' in quantum-many body physics, and can produce \emph{periodic} behavior in the limit of non-terminating sequences~\citep{Cuevas2017IrreducibleFO}. Although we don't give the full details here, such degenerate uMPS can still be converted to equivalent PSR by redefining the observation space to consist of $k$-tuples of adjacent observations, for $k$ the periodicity of the uMPS.}

Now if the top eigenvalue of $\fattau^\dagger$ is $\lambda_* = 1$, $\vec{\sigma}_{t}$ will eventually converge to $\vec{\sigma}_*$, the corresponding fixed point of $\fattau^\dagger$, owing to the second largest eigenvalue of $\fattau^\dagger$ satisfying $|\lambda_2| < 1$. The probability computation in Equation~\ref{app_eq_umps_prob} then reduces to that of a PSR with evaluation functional $\vec{\sigma}_*$. Note that in a PSRs, the $\vec{\sigma}$ is chosen precisely to be the fixed point of the adjoint transfer operator via the normalization requirement $\vec{\sigma}^\dagger \fattau = \vec{\sigma}^\dagger$, forcing $\lambda_* = 1$.

If $\lambda_* \neq 1$ for a uMPS model, we can simply rescale our matrices $\fattau$ to obtain a properly normalized transfer operator, by replacing $\fattau_y$ with $\fattau_y/\lambda_*$. Making this substitution in Equation~\ref{app_eq_umps_prob}, we see that the numerator and denomenator are rescaled by the same constant $\lambda_*^{t-N}$, leaving the overall probability distribution unchanged. As before, this new model produces exactly the same probabilities as a PSR with the evaluation functional $\vec{\sigma}_*$, for $\vec{\sigma}_*$ the fixed point of $\fattau^\dagger$.
\end{proof}

\newpage
\begin{customthm}{4.1}
\label{app_thm_bm_noom}
Non-terminating uniform Born machines are equivalent to norm observable operator models.
\end{customthm}
\begin{proof}
Note that uniform Born machines and norm observable operator models differ only in their evaluation functionals, and that NOOM transfer operators are required to be trace-preserving. While uBMs can have an arbitrary Kraus-rank 1 evaluation functional, NOOMs are restricted to the higher-rank identity evaluation functional $\vec{\mathds{I}}$. Both models have operators of the form $\fattau_y = \phi_y \otimes \phi_y$, which are completely positive with Kraus-rank 1. Therefore, we need to show that an arbitrary uBM is equivalent to one where the transfer operator $\fattau = \sum_y \fattau_y$ is trace-preserving, which is the same as its adjoint $\fattau^\dagger$ being unital, i.e. having the identity $\vec{\mathds{I}}$ as a fixed point~\citep{Nielsen2001QuantumIA}. With this fact demonstrated, the same argument used in the proof of Theorem~\ref{app_thm_umps_psr} to prove convergence of the effective functional of a uMPS to the fixed point $\vec{\sigma}_*^\dagger$ can be applied here. This has the effect of replacing the original Kraus-rank 1 functional of the uBM by the identity $\vec{\mathds{I}}$ in the non-terminating limit, completing the conversion from uBM to NOOM.

If the uBM transfer operator $\fattau$ is already trace-preserving then we are done, so assume it is not. We make the generic assumption\footnote{This is similar to the assumption made in proving Theorem~\ref{app_thm_umps_psr}, and is valid everywhere outside of a measure-zero subset of uBMs. In the case of degenerate eigenvalues, the conversion from uBMs to NOOMs can still be achieved given a slight redefinition of the observation space to combine together $k$ adjacent observations.} that the two eigenvalues of $\fattau$ with greatest magnitude, $\lambda_*$ and $\lambda_2$, satisfy $|\lambda_*| > |\lambda_2|$. By replacing all operators $\phi_y$ by $\phi_y / \sqrt{\lambda_*}$, we convert the uBM into one whose transfer operator $\fattau$ has leading eigenvalue of magnitude 1, a rescaling which leaves the joint probability distributions unchanged. In this case, the quantum Perron-Frobenius theorem~\citep{evans1977} then ensures that $\lambda_* = 1$, and that $\tau^\dagger$ has a unique fixed-point operator $\vec{\sigma_*}$ which is the vectorization of a full-rank (and consequently, invertible) positive definite matrix $\sigma_*$.

A similarity transformation of $S = \sigma_*^{1/2}$ can then be applied to the uBM matrices, replacing $\phi_y$ with $\phi_y' = S \phi_y S^{-1}$. This similarity transformation ensures the new transfer operator $\fattau'$ is trace-preserving, as demonstrated by the following:

\begin{equation}
\begin{split}
\fattau'^\dagger \vec{\mathds{I}} &= \left(\sum_y \overline{\phi_y'}^\dagger \otimes \phi_y'^\dagger\right) \vec{\mathds{I}} = \left(\sum_y (\sigma_*^{-1/2})^T \overline{\phi_y}^\dagger (\sigma_*^{1/2})^T \otimes \sigma_*^{-1/2} \phi_y^\dagger \sigma_*^{1/2} \right) \vec{\mathds{I}} \\
&= \left((\sigma_*^{-1/2})^T \otimes \sigma_*^{-1/2}\right) \left(\sum_y \overline{\phi_y}^\dagger \otimes \phi_y^\dagger \right) \left((\sigma_*^{1/2})^T \otimes \sigma_*^{1/2} \right) \vec{\mathds{I}} = \left((\sigma_*^{-1/2})^T \otimes \sigma_*^{-1/2}\right) \fattau^\dagger \vec{\sigma_*} \\
&= \left((\sigma_*^{-1/2})^T \otimes \sigma_*^{-1/2}\right) \vec{\sigma_*} \\
&= \vec{\mathds{I}}
\end{split}
\end{equation}
In the equations above, we have used the facts that (a) $\sigma_*$, $\sigma_*^{1/2}$, and $\sigma_*^{-1/2}$ are Hermitian, so that complex conjugation acts as $\overline{\sigma_*^{1/2}} = (\sigma_*^{1/2})^T$, (b) $\left( Z^T \otimes X \right) \vec{Y} = \vec{W}$ with $W = XYZ$, for any matrices $X, Y, Z$, and (c) $\fattau^\dagger \vec{\sigma_*} 
= \vec{\sigma_*}$

We have demonstrated that the similarity-transformed uBM now possesses a trace-preserving transfer operator, which by the arguments above ensures it is a valid NOOM in the non-terminating sequence limit.
\end{proof}

\newpage
\begin{customthm}{4.2}
\label{app_thm_hmm_noom}
[HMM $\not \subseteq$ NOOM]
There exist finite-dimensional hidden Markov models that have no equivalent finite-dimensional norm-observable operator model.
\end{customthm}

We first introduce a lemma \citep{ito, vidyasagar2011complete, Thon2015} that will help us in our proof. It tells us that two equivalent PSRs of the same dimension are simply a similarity transform away from one another. 
\begin{lemma}[\citet{Thon2015}]
\label{lemma_equiv_ooms}
Suppose  $(\mathds{C}^n, \vec{\sigma}, \{{\bf \tau}_y\}_{y\in\mathcal{O}}\}, \vec{x}_0)$ and  $(\mathds{C}^n, \vec{\sigma}', \{{\bf \tau}'_y\}_{y\in\mathcal{O}}\}, \vec{x}'_0)$ are two equivalent PSR representations, i.e., they generate the same sequence of probabilities. Then, there exists some non-singular ${\bf S} \in \mathds{C}^{n\times n}$ such that $x'_0 = {\bf S}^{-1}x_0$, $\fattau'_y = {\bf S}^{-1}{\bf \tau}_y{\bf S}$, and $\vec{\sigma}'^T = \vec{\sigma}^T{\bf S}$.
\end{lemma}

\begin{proof}
We construct a class of HMMs for which there are no equivalent NOOMs and give a proof by contradiction. Let $\mathcal{A} = (\mathds{R}^p, \vec{\sigma}, \{\fattau_y\}_{y\in\mathcal{O}}, \vec{x}_0)$ be a minimal PSR equivalent to some hidden Markov model $\mathcal{M} = (\mathds{R}^m, {\mathbf{A}}, {\bf C}, \vec{p}_0)$ for which some future state reachable from the initial state can be written as a convex combination of some previously reached states, i.e., $\vec{x}_k = \alpha \vec{x}_i + \beta \vec{x}_j$ for some $\alpha, \beta > 0$, $\alpha + \beta = 1$ and some $i, j, k \in \mathds{N}$ with $i < j < k$ and $\vec{x}_k = \frac{\fattau_{y_k}\cdots \fattau_{y_1}\vec{x}_0}{\vec{\sigma}^T\fattau_{y_k}\cdots \fattau_{y_1}\vec{x}_0}$ for some sequence of observations $Y = y_1, \ldots, y_k$ (and similarly for $\vec{x}_j$ and $\vec{x}_i$ which truncate the observations at $y_j$ and $y_i$ respectively). 

Suppose there exists an equivalent NOOM (represented in its vectorized form) $\mathcal{M}' = (\mathds{R}^n, \vec{\mathds{I}}, \{{ \boldsymbol \phi_y}\}, \vec{\psi}_0)$. Let $\mathcal{A}' = (\mathds{R}^p, \vec{\sigma}', \{\fattau'_y\}_{y\in\mathcal{O}}, \vec{x}'_0)$ be a minimal PSR computing the same distribution as $\mathcal{M}'$.

Then, by Lemma 1, we have some similarity transform ${\bf S}$ such that $\vec{\sigma}'^T = \vec{\sigma}^T{\bf S}$, $\fattau'_y = {\bf S}^{-1}\fattau_y{\bf S}$, and $\vec{x}'_0 = {\bf S}^{-1}\vec{x}_0$. \citet{Thon2015} show that there are matrices $\Phi, \Pi$ that relate the NOOM to its minimal representation as $\mathcal{A}' = \Pi\Phi^+ \mathcal{M}'\Phi\Pi^+$ (where $+$ represents the Moore-Penrose pseudoinverse). This allows us to relate the NOOM with the minimal PSR representation of its equivalent HMM as $\vec{\mathds{I}}^T\Phi\Pi^+{\bf S}^{-1} = \vec{\sigma}^T$, $\fattau_y = {\bf S}\Pi\Phi^+{\boldsymbol \phi}_y\Phi\Pi^+{\bf S}^{-1}$, and $\vec{x}_0 = {\bf S}\Pi\Phi^+\vec{\psi}_0$.

Now, by the NOOM evolution rules, the NOOM state at timestep $k$ for the sequence $Y$ is $\vec{\psi}_k = \frac{ {\boldsymbol \phi}_{y_k}\cdots {\boldsymbol \phi}_{y_1} \vec{\psi}_0}{ \vec{\mathds{I}}^T {\boldsymbol \phi}_{y_k}\cdots {\boldsymbol \phi}_{y_1} \vec{\psi}_0}$ and the probability of any given observation at that time-step is $P(y|\vec{\psi}_k) = \vec{\mathds{I}}^T{\boldsymbol \phi}_y\vec{\psi}_k$. Further, note that the NOOM state must be a vectorized rank-1 matrix whose eigenvector has unit $\ell_2$ norm, i.e., $\vec{\psi}_k = \text{vec}(\psi_k\psi_k^T)$ with $\|\psi_k\|_2 = 1$.

But we can also write NOOM state in terms of its equivalent PSR-HMM as $$\vec{\psi}_k = \Phi\Pi^+{\bf S}^{-1}\vec{x}_k = \Phi\Pi^+{\bf S}^{-1}(\alpha \vec{x}_i + \beta \vec{x}_j) =  \Phi\Pi^+{\bf S}^{-1}\left(\alpha \left({\bf S}\Pi\Phi^+\vec{\psi}_i\right) + \beta \left({\bf S}\Pi\Phi^+\vec{\psi}_j \right) \right) = \alpha\vec{\psi}_i + \beta \vec{\psi}_j$$

If $\vec{\psi}_i$ and $\vec{\psi}_j$ are valid NOOM states, we have that $\vec{\psi}_k = \alpha\vec{\psi}_i + \beta \vec{\psi}_j = \alpha\text{vec}(\psi_i\psi_i^T) + \beta\text{vec}(\psi_j\psi_j^T) = \text{vec}(\alpha\psi_i\psi_i^T + \beta\psi_j\psi_j^T)$. However, $\alpha\vec{\psi}_i + \beta \vec{\psi}_j$  is not the vectorization of a rank-1 matrix in general. In particular, $\vec{\psi}_k$ has unit Kraus-rank only if $\psi_i$ and $\psi_j$ are linearly dependent, which is true only if $\vec{x}_i$ and $\vec{x}_j$ are linearly dependent. Thus, whenever $\vec{x}_i$ and $\vec{x}_j$ are linearly independent, $\vec{\psi}_k$ cannot have unit Kraus-rank. But as normalized HMM states, $\vec{x}_i$ and $\vec{x}_j$ are always linearly independent, unless they are exactly the same. Hence, a contradiction. Thus, for such an HMM, there is no equivalent NOOM.

\end{proof}

\paragraph{NOOMs are a restrictive model class} The proof above essentially argues that if an HMM had an equivalent NOOM, then anytime a reachable HMM state can be written as a convex combination of some other reachable states, the equivalent NOOM state should also admit a representation as a convex combination of the NOOM-equivalent reachable states, but such a representation violates the condition that NOOM states have unit Kraus-rank, and hence there cannot be an equivalent NOOM. Here, we provide an example of an HMM that can have a state be a linear (here, convex) combination of two prior linearly independent states

\begin{equation}
    \vec{x}_0 = \begin{bmatrix} 1 \\ 0 \end{bmatrix} \qquad {\bf \fattau}_1 = \begin{bmatrix} 0.25 & 0.5 \\ 0.75 & 0 \end{bmatrix} \qquad {\bf \fattau_2} = \begin{bmatrix} 0 & 0 \\ 0 & 0.5\end{bmatrix}
\end{equation}

Then, for the sequence $Y = (1, 1)$:
\begin{equation}
    \vec{x}_0 = \begin{bmatrix} 1 \\ 0 \end{bmatrix} \qquad \vec{x}_1 = \begin{bmatrix} 0.25 \\ 0.75 \end{bmatrix} \qquad \vec{x}_2 = \begin{bmatrix} 0.7 \\ 0.3 \end{bmatrix} = 0.6 \vec{x}_0 + 0.4\vec{x}_1
\end{equation}
Since $\vec{x}_0$ and $\vec{x}_1$ are linearly independent, we know that such an HMM cannot have an equivalent NOOM because of NOOM's state rank constraints. In this case, we see that if an HMM reaches a state that lies strictly inside the convex hull of other reachable states, it rules out a NOOM representation. This constitutes a fairly expansive class of HMMs, suggesting that NOOMs cannot model a wide variety of HMMs, making NOOMs a restrictive model class. This core insight holds for any quantum model that updates and maintains only pure quantum states. Since HQMMs allow states to be vectorizations of arbitrary rank matrices (and so convex combinations of other reachable states), we do not run into the same issues.

\newpage
\begin{customcorollary}{4.1}
[uBM $\nsubseteq$ HMM and HMM $\nsubseteq$ uBM]
There exist finite-dimensional uniform Born machines that have no equivalent finite-dimensional hidden Markov models, and vice-versa
\end{customcorollary}
\begin{proof}
For the first relationship uBM $\nsubseteq$ HMM, note that we already know that NOOM $\nsubseteq$ HMM. \citet{Zhao2010b} demonstrate this by designing the NOOM probability clock model, where the latent state (and therefore conditional probabilities) exhibit oscillatory behavior. The negative entries in NOOM operators allow their top eigenvalues to be complex valued, which allow for such oscillatory behavior. On the other hand, HMMs with non-negative operators with real eigenvalues cannot produce such oscillations. Geometrically, valid finite-dimensional HMM states are restricted to form polyhedral cones, while the oscillations in the probability clock require state dynamics on non-polyhedral (or infinitely generated polyhedral) cones. From Theorem~\ref{app_thm_bm_noom}, we know that a non-terminating uBM with the same operators as the probability clock NOOM will generate identical probabilities, and thus produce the same dynamics as the probability clock NOOM. These are then instances of uBMs that cannot be modeled by a finite dimensional HMM.

For the second relationship HMM $\nsubseteq$ uBM, note that Theorem~\ref{app_thm_hmm_noom} tells us that there exist HMMs that cannot be modeled by finite dimensional NOOMs. The same HMMs cannot be modeled by non-terminating finite dimensional uBMs, as these are equivalent to NOOMs.
\end{proof}

\newpage
\begin{customthm}{5.1}
\label{app_thm_lps_hqmm}
Non-terminating uniform locally purified states are equivalent to hidden quantum Markov models.
\end{customthm}
We follow the same approach as in Theorem~\ref{app_thm_bm_noom}. Uniform LPS models and HQMMs differ only in their evaluation functionals, and that HQMMs operators are trace-preserving. While uLPSs can have an arbitrary evaluation functional, HQMMs are restricted to the identity evaluation functional $\vec{\mathds{I}}$. Both models have operators of the form $\fattau_y = \sum_y \overline{K_y} \otimes K_y$, which are completely positive operators. 

As discussed in Theorem~\ref{app_thm_bm_noom}, the transfer operator $\fattau$ can be rescaled and similarity transformed into one that is trace-preserving. In the limit of non-terminating sequences, the evaluation functional of this transformed model will then converge to $\vec{\mathds{I}}$, the fixed point of $\fattau^\dagger$. Therefore, this similarity transform allows us to map a non-terminating uLPS to an HQMM, proving that non-terminating uLPSs are equivalent to HQMMs.

\newpage
\begin{customcorollary}{5.1}
[NOOM $\subset$ HQMM]
\label{app_hqmm_noom}
Finite dimensional norm-observable operator models are a strict subset of finite dimensional hidden quantum Markov models.
\end{customcorollary}
We know from Theorem~\ref{app_thm_hmm_noom} that there exist HMMs that cannot be modeled by finite-dimensional NOOMs. However, since all HMMs can be modeled by finite dimensional HQMMs, the same HMMs serve as instances of HQMMs that cannot be modeled by NOOMs. This, combined with the fact that NOOM $\subseteq$ HQMM \citep{adhikary2019expressiveness}, give us that NOOM $\subset$ HQMM.

\newpage

\section{Expresiveness of HQMMs (uLPS)} 
\label{app_hqmm_express}
With Corollary~\ref{corollary_hqmm_noom}, we see that HQMMs are a particularly expressive \emph{constructive} class of PSRs that avoid the NPP by design.
What makes HQMMs more expressive than HMMs is that they still admit infinitely generated (i.e., non-polyhedral) cones of valid states (in reference to the discussion in Section \ref{sec_npp}). Geometrically, we can think of the valid unit-Kraus rank (or equivalently pure density matrix) initial NOOM states as the extremal points of a \emph{spectraplex}, which is the intersection of the affine space of unit trace matrices with the convex cone of PSD matrices. The arbitrary Schmidt rank (or mixed density matrices) admitted as initial states for HQMMs fills the entire spectraplex \citep{adhikary2019expressiveness}. 

 The natural question to ask is whether HQMMs can model \emph{any} finite dimensional PSRs, or if we lose any expressiveness in using HQMMs or LPSs? \citet{glasser2019expressive} provide results for the non-uniform case: there exist finite-dimensional non-uniform MPS that have no equivalent finite-dimensional LPS. A similar answer for the \emph{uniform} case is as yet unknown. Investigating this from the HQMM perspective, \citet{monras2016quantum} term the question of which PSRs have equivalent HQMMs as the \emph{completely positive realization problem}. They argue that a necessary and sufficient condition for a PSR to have an equivalent HQMM is if the operators characterizing the PSR come from a semi-definite representable (SDR) cone that also defines the convex cone of valid states.\footnote{The intersection of the cone of positive semi-definite matrices with an affine subspace is called a spectrahedron, and linear maps of spectrahedra are called \emph{spectrahedral shadows} or semi-definitely representable sets. These are the feasible regions of a semidefinite program.} They suggest that if we could show that every valid cone for a PSR satisfied the SDR condition, we could show that HQMMs are equivalent to PSRs via the Helton-Nie conjecture; while it is the case that every SDR set is convex and semi-algebraic, the converse, known as the Helton-Nie conjecture was only recently shown to be false \citep{scheiderer2018spectrahedral}.

Since convex semi-algebraic sets are SDR sets in a wide variety of cases, it may still be the case that HQMMs are equivalent to PSRs. We thus pose two open questions that must be answered to obtain a full characterization of HQMMs relative to PSRs: first, whether the convex cone characterizing the PSR is semi-algebraic, and second, if it satisfies any sufficient conditions for being SDR \citep{helton2009sufficient}, and if not, how common such conditions are. Nevertheless, convex and semi-algebraic cones that are SDR are a broad class and the results thus far show that HQMMs are the most expressive known subset of PSRs. Indeeed, we do not know any examples of PSRs that do not have an equivalent finite-dimensional HQMM.

\newpage
\section{Controlled Stochastic Processes}
\label{app_sec_controlled_stoc_proc}
In this paper, we have discussed various models of stochastic processes, including counter parts in quantum tensor networks and weighted automata. These models are limited in the sense that have no notion of control -- an agent can make observations of a system but cannot perturb it. We now extend our analysis to \emph{controlled} stochastic processes; connections between these models are illustrated in Figure~\ref{fig_controlled_subsets}.

\subsection{Partially Observable Markov Decision Processes}
Partially observable Markov decision processes (POMDPs) can be viewed as natural extensions of hidden Markov models to controlled stochastic processes. Formally, we give the following definition\footnote{We have ignored the reward function commonly associated with controlled stochastic processes to draw a clearer comparison with uncontrolled processes. Our analysis can be equivalently applied with reward functions.}:
\begin{definition}[Partially Observable Markov Decision Processes] An $n$-dimensional partially observable Markov decision process for sets of discrete actions $\mathcal{A}$ and observations $\mathcal{O}$ is a tuple $(\mathds{R}^n, \{\mathbf{A}^{a}\}_{a \in \mathcal{A}}, \{\mathbf{C}^{a}\}_{a \in \mathcal{A}}, \vec{x}_0)$, where the matrices $\mathbf{A}^{a} \in \mathds{R}^{n \times n}_{\geq 0}$ and $\mathbf{C}^{a} \in \mathds{R}^{|\mathcal{O}| \times n}_{\geq 0}$ are non-negative column-stochastic transition and emission matrices corresponding to the action $a$ such that $\vec{1}^{~T} \mathbf{A}^{a} = \vec{1}^{~T} \mathbf{C}^{a} = \vec{1}^{~T}$. The initial state $\vec{x}_0 \in \mathds{R}^{n}_{\geq 0}$ is also non-negative and satisfies $|\vec{x}_0|_1 = 1$.
\end{definition}

POMDPs can also be defined via observable operators $\mathbf{T}^{a}_y = \text{diag}(\mathbf{C}^{a}_{(y,:)}) \mathbf{A}^{a}$ corresponding to taking action $a$ and observing $y$. The state update is done as $\vec{x}_t = \frac{\mathbf{T}_{y}^a~\vec{x}_{t-1}}{\vec{\mathbf{1}}^T \mathbf{T}_{y}^a~\vec{x}_{t-1}}$, and probabilities are calculated as $P(a_1y_1,\dots,a_N,yN) = \vec{\mathbf{1}~}^T \mathbf{T}^{a_k}_{y_k}\dots \mathbf{T}^{a_1}_{y_1} \vec{x}$; exactly the same as HMMs, but with controls.
\subsection{Input-Output Extensions of Stochastic Processes}
In the earlier section, we considered the \emph{uncontrolled} PSRs, which are equivalent to the observable operator models (OOMs) defined in \citet{jaeger2000observable}. Augmenting these models with a notion of control, we obtain input-output OOMs (IO-OOMs) \citep{Jaeger2003DiscretetimeDO} or controlled PSRs \citep{SinghPSR}, which essentially consist of separate OOM for every action. With the same motivation, we can also define analogous input-output HQMMs as follows
\begin{definition}[Input-Output HQMM] An $n^2$-dimensional input-output hidden quantum Markov model for sets of discrete actions $\mathcal{A}$ and observations $\mathcal{O}$ is a controlled stochastic process given by the tuple ($\mathds{C}^{n^2}, \vec{\mathds{I}}, \{\mathbf{L}_{y}^a\}_{y \in \mathcal{O}, a \in \mathcal{A}}, \vec{\rho}_0$). The set of operators $\{\mathbf{L}_y^a\}_{y \in \mathcal{O}}$ for every action $a$ form a set of CP-TP Liouville operators as in Definition~\ref{def:hqmm}. The initial state $\vec{\rho}_0$ is a vectorized unit-trace Hermitian PSD matrix of arbitrary rank.
\end{definition}
The state update and probability computations for IO-HQMMs are done exactly the same as with HQMMs, but now with operators indexed by both observations and actions at each time step: $P(a_1 y_1 \dots a_k y_k) = \vec{\mathds{I}}^T \mathbf{L}_{y_k}^{a_k} \dots \mathbf{L}_{y_1}^{a_1} \vec{\rho}_0$ and $\vec{\rho}~' = \frac{\mathbf{L}_{y_k}^{a_k} \dots \mathbf{L}_{y_1}^{a_1} \vec{\rho}_0}{\vec{\mathds{I}}^T \mathbf{L}_{y_k}^{a_k} \dots \mathbf{L}_{y_1}^{a_1} \vec{\rho}_0}$

Similarly, we can define input-output uLPSs and input-output uMPSs in exactly the same way, with non-terminating variants being equivalent to IO-HQMMs and IO-OOMs/PSRs respectively. When viewed as tensor network in Figure~\ref{fig_controlled_subsets}, we immediately see that such IO-HQMMs and uniform IO-LPSs have the same structure as matrix product operators \citep{Chan2016MatrixPO, Murg2008MatrixPO} which can be viewed as MPSs with an additional open index at each core.
\paragraph{Quantum Observable Markov Decision Processes are Input-Ouput HQMMs}
\citet{qomdp} developed quantum observable Markov decision processes (QOMDPs) as a strict generalization of classical POMDPs, by swapping the belief state vector $\vec{x}$ with the density matrix of a quantum state $\rho$. Ignoring reward functions, QOMDPs can be defined as follows
\begin{definition}
An $n$-dimensional quantum observable Markov decision process for sets of discrete actions $\mathcal{A}$ and observations $\mathcal{O}$ is a tuple $(\mathds{C}^n, \{\mathbf{K}^a_{y}\}_{a \in \mathcal{A}, y \in \mathcal{O}}, \rho_0)$, where the set of operators $\{\mathbf{K}^a_{y} \in \mathds{C}^{n \times n}\}_{y \in \mathcal{O}}$ for each action $a$ forms a quantum channel with $\sum_{y \in \mathcal{O}} \mathbf{K}_{y}^{a\dagger}\, \mathbf{K}_{y}^a = \mathds{I}$. The initial state $\rho_0$ is a unit-trace Hermitian PSD density matrix of arbitrary rank.

\end{definition}
The probability of an action-observation sequence $a_1 y_1 \cdots a_T y_T$ for a QOMDP is then taken to be $\text{Tr}\left( K_{y_k}^{a_k}\dots K_{y_1}^{a_1}~\rho_0~{K_{y_1}^{a_1}}^\dagger\dots {K_{y_k}^{a_k}}^\dagger \right)$. To draw connections with HQMMs, we represent the vectorized version of the QOMDP update equations as follows:
\begin{small}
\begin{align}
    &\vec{\rho}~' = \text{vec}\left(\frac{\mathbf{K}_{o_k}^{a_k}\dots \mathbf{K}_{o_1}^{a_1}~\rho~{\mathbf{K}_{o_1}^{a_1}}^\dagger\dots {\mathbf{K}_{o_k}^{a_k}}^\dagger}{\text{Tr}\left( \mathbf{K}_{o_k}^{a_k}\dots \mathbf{K}_{o_1}^{a_1}~\rho~{\mathbf{K}_{o_1}^{a_1}}^\dagger\dots {\mathbf{K}_{o_k}^{a_k}}^\dagger \right)}\right)
    = \frac{\left({\overline{\mathbf{K}}_{y_T}^{a_T}} \otimes {\mathbf{K}_{y_T}^{a_T}}\right) \dots \left({\overline{\mathbf{K}}_{y_1}^{a_1}} \otimes {\mathbf{K}_{y_1}^{a_1}}\right) \vec{\rho}}{\vec{\mathbb{I}}\left({\overline{\mathbf{K}}_{y_T}^{a_T}} \otimes {\mathbf{K}_{y_T}^{a_T}}\right) \dots \left({\overline{\mathbf{K}}_{y_1}^{a_1}} \otimes {\mathbf{K}_{y_1}^{a_1}}\right) \vec{\rho}}
    = \frac{\mathbf{L}_{y_T}^{a_T}\dots\mathbf{L}_{y_1}^{a_1}\vec{\rho}_0}{\vec{\mathds{I}}^T \mathbf{L}_{y_T}^{a_T}\dots\mathbf{L}_{y_1}^{a_1}\vec{\rho}_0}
    \label{eq_qomdp_updates}
\end{align}
\end{small}
Notice that this expression is identical to the update equation for IO-HQMMs. The only difference is that
the QOMDP operators are restricted to have unit Kraus-rank: $\mathbf{L}_y^a = {\overline{\mathbf{K}}_y^a} \otimes \mathbf{K}_y^a$. This makes them somewhat alike NOOMs; however, unlike NOOMs, QOMDPs allow latent states of arbitrary Schmidt-ranks -- i.e. both mixed and pure states. We thus arrive at the following theorem characterizing the expressiveness of QOMDPs
\begin{theorem}
QOMDPs $\subseteq$ IO-HQMMS $=$ non-terminating uniform IO-LPS $\subseteq$ IO-OOMs = PSRs 
\end{theorem}
Quantum Markov decision processes were also studied by \citet{YingYing}, who arrived at essentially the same model as QOMDPs. \citet{Cidre2016} proposed an alternate formulation, called Quantum MDPs (QuaMDPs), along with an associated point-based value iteration algorithm similar to that used to learn classical POMDPs \citep{Pineau2003PointbasedVI}. Although connections to QOMDPs were not considered in the original work, QuaMDPs are a special case of QOMDPs. The filtering process in IO-HQMMs and QOMDPs corresponds to positive operator valued measurements, which are generalizations of the more restrictive projection valued measurements used in QuaMDPs. One can reduce positive operator valued measurements to projection based measurements via the Stinesprig dilation theorem \citep{Stinespring1955PositiveFO}, but this requires the system interacting with an ancillary sub-system; this is not the case in QuaMDPs.

\paragraph{Undecidability of Perfect Planning}
In moving from POMDPs to QOMDPs, \citet{qomdp} find an apparent classical-quantum separation in the problem of perfect goal state reachability. They consider particular instances of controlled processes where certain states are labeled as goals, and are set to be absorbing -- the transition probabilities from these states to any other state is zero; we will refer to these as \emph{goal oriented} models. 

In such models, the perfect goal state reachability problem is stated as follows: given an arbitrary initial state and a goal state, is there a sequence of operators that will leave the system in a goal state with probability $1$ in a finite number of steps. In other words: is there a policy that will take the agent deterministically from some initial state to a goal state in a finte number of steps? \citet{qomdp} show that this problem is undecidable for QOMDPs, even though it is decidable for POMDPs. 
Formally, the undecidability of perfect goal state reachability in QOMDPs is a consequence of the undecidability of the quantum measurement occurence problem \citep{Eisert2012QuantumMO}. Intuitively, \citet{qomdp} point out that the decidability of the same problem for POMDPs boils down to its non-negativity constraints. However, negative parameters are also present in \emph{classical} generalizations of POMDPs, namely IO-OOMs or PSRs, as well as other quantum models including IO-HQMMs or IO-LPS. Thus, we would expect this problem should be undecidable for these models as well. Indeed, we can show this to be the case using the subset relationships we have established between these models and QOMDPs.
\begin{theorem}[Perfect Goal State Reachability]
Given an initial state and a goal state, it is undecidable whether there exists a policy that will leave a goal-oriented IO-HQMMs, (non-terminating uniform) IO-LPS, IO-OOM or PSR in the goal state in a finite number of steps.
\end{theorem}
\begin{proof}
QOMDPs are contained within the class of IO-HQMMs or IO-uLPS, which are, in turn, contained within the class of IO-OOMs or PSRs. If a problem is undecidable for a model class, it must be undecidable in general for any other class that contains it as a subset. Therefore, we immediately see that the undecidability carries over to these models as well.
\end{proof}

\begin{figure}
\begin{subfigure}{0.45\textwidth}
\includegraphics[scale=0.35]{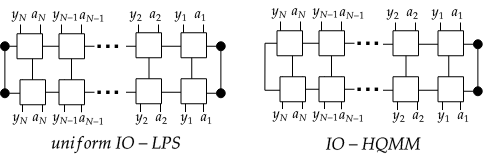}
\end{subfigure}
\begin{subfigure}{0.5\textwidth}
\includegraphics[scale=0.35]{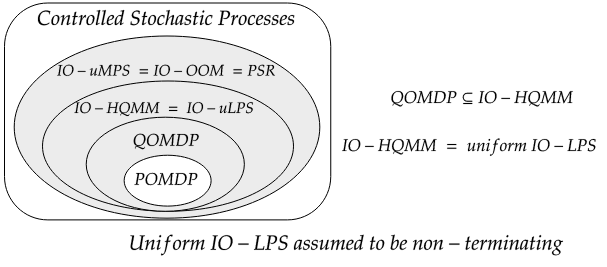}
\end{subfigure}
\caption{(Left) Tensor network diagrams (Right) Subset relationships between models, and new relationships established in this paper. The grey sections represent potentially non-strict subsets.}
\label{fig_controlled_subsets}
\end{figure}

\end{appendix}

\end{document}